\documentclass[11pt, colorlinks, linkcolor=blue,pagebackref]{article}
\usepackage[utf8]{inputenc}
\usepackage[margin=1in]{geometry}
\usepackage{amsmath,amsfonts,amsthm,amssymb}
\usepackage[numbers]{natbib}
\usepackage{mathtools,cancel}
\usepackage{color}

\usepackage{amsmath,amsfonts,amsthm,amssymb}
\usepackage{mathtools}
\usepackage{algorithm, algorithmic, xspace}

\usepackage{thmtools,thm-restate}

\declaretheorem[parent=section]{theorem}
\declaretheorem[sibling=theorem]{lemma}

\declaretheorem[sibling=theorem]{claim}

\declaretheorem[sibling=theorem,style=definition]{definition}

\newtheorem*{theorem*}{Theorem}
\newtheorem*{lemma*}{Lemma}

\declaretheoremstyle[
        spaceabove=\topsep, 
        spacebelow=\topsep, 
        bodyfont=\normalfont,
        notefont=\normalfont\bfseries,
        notebraces={}{},
        qed=$\blacksquare$, 
    ]{proofstyle}
\declaretheorem[style=proofstyle,numbered=no,name=Proof]{proof}

\usepackage[extdef=true]{delimset}
\definecolor{dgreen}{rgb}{0,0.5,0}
\usepackage[colorlinks = true,
            linkcolor = blue,
            urlcolor  = blue,
            citecolor = dgreen,
            anchorcolor = blue]{hyperref}

\usepackage[capitalise]{cleveref}
\usepackage{crossreftools}

\crefname{claim}{Claim}{Claims}

\newcommand{\E}{\mathop{\mathbb{E}}}

\newcommand{\cH}{\mathcal{H}}

\newcommand{\cF}{\mathcal{F}}
\newcommand{\cX}{\mathcal{X}}
\DeclareMathOperator{\Ldim}{Ldim}

\newcommand{\F}{\mathcal{F}}

\newcommand{\Z}{\mathcal{Z}}

\newcommand{\ignore}[1]{}

\newcommand{\s}[2]{{#1^{(#2)}}}

\newcommand{\sib}{\mathrm{s}}

\renewcommand{\H}{\mathcal{H}}
\newcommand{\V}{\mathcal{V}}
\renewcommand{\P}{\mathop{\mathbb{P}}}
\newcommand{\Tc}{\mathcal{G}_{c}}
\newcommand{\X}{\mathcal{X}}

\newcommand{\hist}{\proc{hist}}
\newcommand{\freq}{\mathrm{freq}}
\newcommand{\efreq}{\overline{\mathrm{freq}}}
\newcommand{\proc}[1]{{\fontfamily{lmtt}\selectfont #1}\xspace}

\newcommand{\counter}{\mathrm{counter}}
\newcommand{\svt}{\proc{sparse}}
\newcommand{\aboveT}{\proc{Above-threshold}}
\newcommand{\svto}{\proc{HistSparse}}
\newcommand{\dpsoa}{\proc{DP-SOA}}
\newcommand{\soa}{\proc{SOA}}
\newcommand{\lap}{\mathrm{LAP}}
\newcommand{\constfreq}{\Theta_{(\ref{eq:Gstability})}(k_1,T,\beta)}
\newcommand{\consthistsparse}{\Theta_{(\ref{eq:histsparse})}(c,\eta, T,\beta,\epsilon,\delta)}
\newcommand{\consteta}{
\frac{2^{-4k_1}}{4k_1}}
\newcommand{\consthist}{\Theta_{(\ref{eq:hist})}(\eta,\beta,\epsilon,\delta)}
\newcommand{\constsparse}{\Theta_{(\ref{eq:sparse})}(c,\alpha,\beta,\epsilon,\alpha)}
\newcommand{\constc}{4k_1/\eta} 

\newcommand{\ep}{\epsilon}
\newcommand{\MW}{\mathcal{W}}
\newcommand{\MX}{\mathcal{X}}
\newcommand{\ME}{\mathcal{E}}
\newcommand{\myparagraph}[1]{\smallskip\noindent\textbf{#1}}
\newcommand{\fref}[1]{\cref{#1}}

\title{Littlestone Classes are Privately Online Learnable}
 \author{
 Noah Golowich\thanks{MIT EECS, Cambridge, MA.  \texttt{nzg@mit.edu}.} 
 \and Roi Livni\thanks{Department of Electrical Engineering, Tel Aviv University; \texttt{rlivni@tauex.tau.ac.il}.} 
 }

\date{\today}

\begin{document}
\maketitle

\begin{abstract}
    We consider the problem of online classification under a privacy constraint. In this setting a learner observes sequentially a stream of labelled examples $(x_t, y_t)$, for $1 \leq t \leq T$, and returns at each iteration $t$ a hypothesis $h_t$ which is used to predict the label of each new example $x_t$. The learner's performance is measured by her regret against a known hypothesis class $\mathcal{H}$.
    We require that the algorithm satisfies the following privacy constraint: the sequence $h_1, \ldots, h_T$ of hypotheses output by the
    algorithm needs to be an $(\epsilon, \delta)$-differentially private function of the whole input sequence $(x_1, y_1), \ldots, (x_T, y_T)$.
    We provide the first non-trivial regret bound for the realizable setting. Specifically, we show that if the class $\mathcal{H}$ has constant Littlestone dimension then, given an oblivious sequence of labelled examples, there is a private learner that makes in expectation at most $O(\log T)$ mistakes -- comparable to the optimal mistake bound in the non-private case, up to a logarithmic factor. Moreover, for general values of the Littlestone dimension $d$, the same mistake bound holds but with a doubly-exponential in $d$ factor. 
    A recent line of work has demonstrated a strong connection between classes that are online learnable and those that are differentially-private learnable. Our results strengthen this connection and show that an online learning algorithm can in fact be directly privatized (in the realizable setting).
    We also discuss an adaptive setting and provide a sublinear regret bound of $O(\sqrt{T})$.
\end{abstract}

\section{Introduction}
Privacy-preserving machine learning has attracted considerable attention in recent years, motivated by the fact that individuals' data is often collected to train statistical models, and such models can leak sensitive data about those individuals \citep{dr14,kearns_ethical_2019}. The notion of \emph{differential privacy} has emerged as a central tool which can be used to formally reason about the privacy-accuracy tradeoffs one must make in the process of analyzing and learning from data. A considerable body of literature on \emph{differentially private machine learning} has resulted, ranging from empirical works which train deep neural networks with a differentially private form of stochastic gradient descent \citep{abadi_deep_2016}, to a recent line of theoretical works which aim to characterize the optimal sample complexity of privately learning an arbitrary hypothesis class \citep{alon_private_2019,bun20equivalence,ghazi2020sample}.

Nearly all of these prior works on differentially private learning, however, are limited to the statistical learning setting (also known as the \emph{offline} setting): this is the setting where the labeled data, $(x_t, y_t)$, are assumed to be drawn i.i.d.~from some unknown population distribution. This setting, while very well-understod and readily amenable to analysis, is unlikely to hold in practice. Indeed, the data $(x_t, y_t)$ fed as input into the learning algorithm may shift over time (e.g., as a consequence of demographic changes in a population), or may be subject to more drastic changes which are \emph{adaptive} to the algorithm's prior predictions (e.g., drivers' reactions to the recommendations of route-planning apps may affect traffic patterns, which influence the input data to those apps). For this reason, it is desirable to develop provable algorithms which make fewer assumptions on the data.

 In this work, we do so by studying the setting of \emph{(private) online learning}, in which the sequence of data $(x_t, y_t)$ is allowed to be arbitrary, and we also discuss a certain notion of privacy in a setting where it is even allowed to adapt to the algorithm's predictions in prior rounds. We additionally restrict our attention to the problem of classification, namely where the labels $y_t \in \{0,1\}$; thus we introduce the problem of differentially private online classification, and prove the following results (see  \cref{sec:opl-setup} for the exact setup):

\begin{itemize}
\item In the realizable setting with an oblivious adversary, we introduce a private learning algorithm which, for hypothesis classes of Littlestone dimension $d$ (see \cref{sec:prelim:online}) and time horizon $T$, achieves a mistake bound of $\tilde O(2^{O(2^d)} \cdot \log T)$, ignoring the dependence on privacy parameters (\cref{thm:main:oblivious}).
\item In the realizable setting with an adaptive adversary, we show that a slight modification of the above algorithm achieves a mistake bound of $\tilde O(2^{O(2^d)} \cdot \sqrt{T})$ (\cref{thm:main:adaptive}). 
\end{itemize}
We remark that no algorithm (even without privacy, allowing randomization, and in the oblivious adversary setting) can achieve a mistake bound of smaller than $\Omega(d)$ for classes of Littlestone dimension $d$ \citep{littlestone1988learning,shalev2011online}. Therefore, a class of infinite Littlestone dimension cannot have any finite mistake bound, and the regret for any algorithm, for any time horizon $T$, is $\Omega(T)$.
Thus, our results listed above, which show a mistake-bound (which is also the regret in the realizable setting) of $\tilde O_d(\sqrt{T})$ for classes of Littlestone dimension $d$, establish that in the realizable setting, finiteness of the Littlestone dimension is necessary and sufficient for online learnability (\cite{rakhlin_online_2015}) with differential privacy.

Recently it was shown by \citet{alon_private_2019} and \citet{bun20equivalence} (later to be improved by \citet{ghazi2020sample}) that finiteness of the Littlestone dimension is necessary and sufficient for private learnability in the \emph{offline} setting, namely with i.i.d.~data (and both in the realizable and agnostic settings). Since, as remarked above, the Littlestone dimension characterizes online learnability (even without privacy), this means that a binary hypothesis class is privately (offline) learnable if and only if it is online learnable. Our result thus strengthens this connection, showing that the equivalence also includes private \emph{online} learnability (in the realizable setting). 

\subsection{Related work}
A series of papers \citep{dwork_differential_2010,jain_differentially_2012,thakurta_nearly_2013,dwork_analyze_2014,agarwal_price_2017} has studied the problem of diferentially private online convex optimization, which includes specific cases such as private prediction from expert advice and, when one assumes imperfect feedback, private non-stochastic multi-armed bandits \citep{tossou_algorithms_2016,tossou_achieving_2017,ene_projection_2020,hu_optimal_2021}. These results show that in many regimes privacy is free for such problems: for instance, for the problem of prediction from the expert advice (with $N$ experts), \citet{agarwal_price_2017} shows that an $\ep$-differentially private algorithm (based on follow-the-regularized-leader) achieves regret of $O \left( \sqrt{T}+ \frac{N \log^2 T}{\ep} \right)$, which matches the non-private regret bound of $O(\sqrt{T \log N})$ when $T \geq \tilde \Omega((N/\ep)^2)$. Our results can be seen as extending such ``privacy is (nearly) free'' results to the nonparametric setting where we instead optimize over an arbitrary class of finite Littlestone dimension. Our techniques are different from those of the above papers. 

In addition to \cite{bun20equivalence,ghazi2020sample} which establish private learning algorithms for classes with finite Littlestone dimension in the i.i.d.~(offline) setting, there has been an extensive line of work on private learning algorithms in the offline setting: \cite{kasiviswanathan2008what,beimel_characterizing_2019,beimel_bounds_2014,feldman_sample_2014} study the complexity of private learning with \emph{pure differential privacy}, \cite{kaplan_privately_2020, bun_differentially_2015,bun_composable_2018,beimelNissimStemmer2013} study the sample complexity of privately learning thresholds, and \cite{kaplan_private_2020,kaplan_how_2020,beimel_private_2019} study the sample complexity of privately learning halfspaces.

\section{Preliminaries}\label{sec:prelim}
In this section we introduce some background concepts used in the paper. 
\subsection{Online Learning}\label{sec:prelim:online}
We begin by revisiting the standard setting of online-learning:
We consider a sequential game between a \emph{learner} and an \emph{adversary}. Both learner and adversary know the sets $\cX$ and $\cH$. The game proceeds for $T$ rounds (again $T$ is known) and at each round $t\le T$, the adversary chooses a pair $(x_t,y_t)$ and presents the learner with the example $x_t$. The learner then must present the adversary with a hypothesis (perhaps randomly) $h_t : \cX \to \{0,1\}$.  $h_t$ is not required to lie in $\cH$\footnote{This setup is known as the \emph{improper} learning problem. In the \emph{proper} version of the problem, it is required that $h_t \in \cH$ and we leave a study of proper private online learning for future work. (see \cite{hanneke2021online} for a discssion on proper online learning in the non-private case}. Finally the adversary presents the learner with $y_t$, which the learner uses to update its internal state. The performance of the learner is measured by its regret which is its number of mistake vs. the optimal decision in hindsight:
     \begin{equation}
        \label{eq:reg-defn}
        \E\left[\sum_{t=1}^T 1[h_t(x_t) \neq y_t] - \min_{h^\star \in \cH} \sum_{t=1}^T 1[h^\star(x_t) \neq y_t]\right].
        \end{equation}

The adversary is said to be \emph{realizable} if it presents the learner with a sequence of examples $(x_t, y_t)$ so that there is some $h^\star \in \cH$ so that for each $t \in [T]$, $h^\star(x_t) = y_t$. In the realizable setting, the regret simply counts the number of mistakes the learner makes. And we measure the performance by its \emph{mistake bound}, namely the maximum, over all possible realizable adversaries, of
\[\E\left[\sum_{t=1}^T 1[h_t(x_t) \neq y_t]\right].\] 
In the setting with an \emph{agnostic} adversary, we do not require such $h^\star$ to exist; and we measure the  learner by its (worst-case) \emph{regret}, as in \cref{eq:reg-defn}.
In this paper we focus on the  realizable setting; the (private ) agnostic setting is left as an interesting direction for future work.  

Additionally, we normally make a distinction between two types of adversaries: An \emph{oblivious} adversary chooses its sequence in advance and at each iteration $(x_t,y_t)$ is revealed to the learner. In the \emph{adversarial} setting, the adversary may choose $(x_t,y_t)$ as a function of the learner's previous choices: i.e. $h_1,\ldots, h_{t-1}$.  This definition follows the standard setup of online learning (see \cite{cesa2006prediction} for example). We note though, that in the non-private setting of online binary classification, one can obtain results against an adversary that even gets to observe the learner's prediction at time-step $t$. However, we will simplify here by considering the more standard setting. It is interesting to find out if we can compete against such a strong adversary in the private setup.

\myparagraph{Littlestone dimension}
We next turn to introduce the Littlestone dimension which is a combinatorial measure that turns out to characterize learnability in the above setting.

Let $\cH$ be a class of hypotheses $h : \MX \to \{0,1\}$. To define the Littlestone dimension of $\cH$, we first introduce \emph{mistake trees}: a mistake tree of \emph{depth} $d$ is a complete binary tree, each of whose non-leaf nodes $v$ is labeled by a point $x_v \in \MX$, and so that the two out-edges of $v$ are labeled by $0$ and $1$. We associate each root-to-leaf path in a mistake tree with a sequence $(x_1, y_1), \ldots, (x_d, y_d)$, where for each $i \in [d]$, the $i$th node in the path is labeled $x_i$ and the path takes the out-edge from that node labeled $y_i$. A mistake tree is said to be \emph{shattered} by $\cH$ if for any root-to-leaf path whose corresponding sequence is $(x_1, y_1), \ldots, (x_d, y_d)$, there is some $h \in \cH$ so that $h(x_i) = y_i$ for  all $i \in [d]$. The Littlestone dimension of $\cH$, denoted $\Ldim(\cH)$, is the depth of the largest mistake tree that is shattered by $\cH$. 

\myparagraph{The Standard Optimal Algorithm (\soa)}
Suppose $\cH$ is a binary hypothesis class with Littlestone dimension $d$. Littlestone \cite{littlestone1988learning} showed that there is an algorithm, called the \emph{Standard Optimal Algorithm} (\soa), which, against an adaptive and realizable adversary, has a mistake bound of $d$; moreover, this is the best possible mistake bound. We will access the \soa as a black box. The underlying assumption we make is that given a realizable sequence $(x_1,y_1),\ldots, (x_T,y_T)$, the $\soa$ makes at most $\Ldim(\cH)$ mistakes. We will also assume that whenever the algorithm $\soa$ makes a mistake then it changes it state: namely, if the algorithm makes mistake on example $t$ then $h_{t+1}\ne h_t$, this is in fact true for the \soa algorithm, but it can be seen that any algorithm with mistake bound can be modified to make sure this holds (simply by reiterating the mistake until the algorithm does change state). We refer the reader to \cite{littlestone1988learning,shalev2011online} for the specifics of it.
\subsection{Differential Privacy}
We next recall the standard notion of $(\epsilon,\delta)$--differential privacy:
\begin{definition}[Differential privacy]
  \label{def:dp}
  Let $n$ be a positive integer, $\ep, \delta \in (0,1)$, and $\MW$ be a set. A randomized algorithm $A : (\MX \times \{0,1\})^n \to \MW$ is defined to be \emph{$(\ep, \delta)$-differentially private} if for any two datasets $S, S' \in (\MX \times \{0,1\})^n$ differing in a single example, and any event $\ME \subset \ME$, it holds that
  \begin{align}
\Pr[A(S) \in \ME] \leq e^\ep \cdot \Pr[A(S') \in \ME] + \delta\nonumber.
  \end{align}
\end{definition}   
\myparagraph{Adaptive Composition} The online nature of the problem naturally requires us to deal with adaptive mechanisms that query the data-base. We thus depict here the standard framework of adaptive querying, and we refer the reader to \citet{dr14} for a more detailed exposition. 

In this framework we assume a sequential setting, where at step $t$ an adversary chooses two adjacent datasets $S_t^1$ and $S_t^0$, and a mechanism $M_t(S)$ from a class $\cF$ and receives $z_t^b = M_t(S_t^b)$ for some $b\in \{0,1\}$ (where $b$ does not depend on $t$).
 \begin{definition}
 We say that the family $\cF$ of algorithms over databases satisfies \emph{$(\epsilon,\delta)$-differential privacy under $T$-fold adaptive composition} if for every adversary $A$ and event $\ME$, we have 
 \[ \Pr((z_1^0,\ldots, z^0_T)\in \ME)\le e^{\epsilon} \Pr((z_1^1,\ldots, z^1_T)\in \ME)+\delta.\]
 \end{definition}

\section{Problem Setup}\label{sec:opl-setup}
We now formally introduce the main problem considered in this paper, namely that of \emph{private online learning}. Let $\cX$ be a set, and let $\cH$ be a set of \emph{hypotheses}, namely of functions $h : \cX \to \{0,1\}$. 
 We consider the setting depicted in \cref{sec:prelim:online} and in this framework we want to study the learnability of \emph{private} learners which are defined next. We make a distinction between the case of an oblivious and an adaptive adversary:

    \myparagraph{Private online learning vs. an \emph{oblivious} adversary} As discussed, in this setting the adversary must choose the entire sequence $(x_1, y_1), \ldots, (x_T, y_T)$ before its interaction with the learner (though it may use knowledge of the learner's \emph{algorithm}). In particular, the samples $(x_t, y_t)$ do not depend on any random bits used by the learner.
 Thus, in the \emph{private online learning problem} we merely require that the sequence of hypotheses $(h_1, \ldots, h_T)$ output by the learner is $(\epsilon, \delta)$-differentially private as a function of the entire input sequence $(x_1, y_1), \ldots, (x_T, y_T)$.
  
\myparagraph{Private online learning vs. an \emph{adaptive} adversary:} In the adaptive setting, the adversary may choose each example $(x_t, y_t)$ as a function of all of the learner's hypotheses up to $t$.
This makes the notion of privacy a little bit more subtle, so we need to carefully define what we mean here by $(\epsilon,\delta)$-privacy. We consider then the following scenario:

At each round $t$, the adversary outputs two outcomes $(x^0_t,y^0_t)$ and $(x^1_t,y^1_t)$. The learner then outputs $h^b_t$ and $(x_t^b,y_t^b)$ is revealed to the learner where $b\in\{0,1\}$ is independent of $t$. We require that the sequences $S_T^0=\{(x^0_t,y^0_t)\}$ and $S_T^1=\{(x^1_t,y^1_t)\}$ differ in, at most, a single example.
We will say that an adaptive online classification algorithm is $(\epsilon,\delta)$ differentially private, if for any event $\ME$ and any adversary, it holds that
\[\Pr[(h^1_1,\ldots, h^1_T) \in \ME] \leq e^\ep \cdot \Pr[(h^0_1,\ldots, h^0_T) \in \ME] + \delta\nonumber.
  \]

The notion is similar to privacy under $T$-fold adaptive composition. 
Normally, though, for a mechanism to be $(\epsilon,\delta)$-differentially private under $T$-fold adaptive compositions, \citet{dwork2014algorithmic} requires it to be private under an adversary that may choose at each iteration \emph{any} two adjacent datasets, $S_i^0$, $S^1_i$.  Note, however that, in the online setup, the utility is dependent only on a single point at each iteration, hence such a requirement will be too strong (in fact, the learner will then be tested on two arbitrary sequences).


\section{Main Results}\label{sec:results}
We next state the main results of this paper, we start with a logarithmic regret bound for realizable oblivious learning.

\begin{theorem}[Private Oblivious online-learning]\label{thm:main:oblivious}
For a choice of $k_1=\tilde{O}(2^{d+1})$, and
\[k_2 =\tilde{O}\left(\frac{2^{8\cdot 2^d}}{\epsilon}\ln T/\delta \right),\]

Running \dpsoa (\cref{alg:stablesoa}) for $T$ iterations on any realizable sequence $(x_1,y_1),\ldots, (x_T,y_T)$, the algorithm outputs a sequence of predictors $h_1,\ldots, h_T$ such that
\begin{itemize}
    \item The algorithm is $(\epsilon,\delta)$ differentially private.
    \item The expected number of mistakes the algorithm makes is
    \[\E[\sum_{t=1}^T h_t(x_t)\ne y_t]=\tilde{O}\left(\frac{2^{8\cdot 2^d}}{\epsilon}\ln T/\delta\right).\]
\end{itemize}
\end{theorem}

\cref{thm:main:oblivious} shows that, up to logarithmic factor, the number of mistakes in the private case is comparable with the number of mistakes in the non-private case, when $d$ the Littlestone dimension of the class is constant. We obtain, though, a strong deterioration in terms of the Littlestone dimension -- sublinear dependece vs. double exponential dependence. As discussed, \citet{ghazi2020sample} improved the dependence in the batch case to polynomial, and it remains an open question if similar improvement is applicable in the online case.
We next turn to the adversarial case
\begin{theorem}[Private Adaptive online-learning]\label{thm:main:adaptive}
There exists an adaptive online classification algorithm that is $(\epsilon,\delta)$-differentially private with expected regret over a realizble seqeunce:
\[\E[\sum_{t=1}^T h_t(x_t)\ne y_t]=\tilde{O}\left(  \frac{2^{O(2^d)}\sqrt{T}\log 1/(\delta)}{\epsilon}\right).\]
\end{theorem}
\cref{thm:main:adaptive} provides a sublinear regret bound, which is in fact optimal for the agnostic case. However, in the non-private (realizable) case it is known that constant regret can be obtained\footnote{and as discussed, the adversary may even depend on $h_t$ at round $t$} . We leave it as an open problem whether one can achieve logarithmic regret in the realizable adaptive setting.
\section{Algorithm}\label{sec:alg}
We next present our main algorithm for an oblivious, realizable online private learning algorithm.
\begin{algorithm}
\begin{algorithmic}
\STATE \textrm{Input} $(\epsilon,\delta)$, $k_1,k_2$. 
\STATE Set $\eta=\consteta$, and $c=\constc$
\STATE Let $G=(V,E)$ be a forest of $k_2$ full binary trees, each with $k_1$ leaves.
\STATE Let $\pi:T\to \mathrm{Leaves}(V)$ be a random mapping that maps $t\in [T]$ to a random leaf. 
\STATE Set $S_v=\emptyset$ for each leaf $v$ and $S_u=\perp$ for each non-leaf vertex $u$ (where we define $A(\perp)=\perp$). 
\STATE Initialize $\V_1$ to be the set of all leaves in the forest.
\STATE set $\s{v_1}{i}$ be an arbitrary leaf from the tree $G_i$, for each $i\in [k_2]$
\FOR{t=1 to T}
    \STATE Run $\svto_{\epsilon,\delta,\eta,c}(h_{t-1},L_t)$ on the List $L_t=\{A(S_{\s{v_t}{i}})\}_{i=1}^{k_2}$ and receive $h_t$
    \STATE Predict $h_t(x_t)=\hat y_t$, and observe $y_t$.
        \STATE Choose $v_1\in \V_t$ to be an antecedent of leaf $\pi(t)$ \%\small{there exists a unique antecedent in $\V_t$}\normalsize
        \STATE Set $v_2=\sib(v_1)$ (if $v_1$ is the root, continue to the next iteration).
        \STATE Set $(S_{v_1},(x_t,y_t)) \to S_{v_1}$.
        \WHILE{$A(S_{v_1})\ne A(S_{v_2})$ AND $v_1,v_2 \in \V_t$}
            \STATE Set $\bar v$ to be the parent of $v_1,v_2$
            \STATE Choose an arbitrary $x_{\bar v}$ such that $A(S_{v_1})[x_{\bar v}]\ne A(S_{v_2})[x_{\bar v}]$ and $y_{\bar v}$ randomly
            \STATE Set $(S_{v_i},(x_{\bar v},y_{\bar v}))\to S_{\bar v}$ where $i$ is such that $A(S_{v_i})[x_v]\ne y_v$.
            \STATE Remove $v_1$,$v_2$ from $\V_t$ and add $\bar v$ to $\V_t$.
                \STATE Let $v_1$ be $\bar v_t$
                \IF{$v_1$ is not the root}
                    \STATE Set $v_2$ to be the sibling of $v_1$
                \ELSE
                    \STATE Set $v_1=v_2$ (and hence exit the loop.)
                \ENDIF
        \ENDWHILE
        \IF{The While loop was executed at least once}
        \STATE Let $i$ be the tree for which $\pi(t)$ belongs to.
        \STATE Choose randomly a vertex $v$ in tree $i$ such that $v,\sib(v)\in \V_t$ and $A(S_v)=A(S_{\sib(v)})$ (break ties by choosing randomly). 
        \STATE (If no such $v$ exists, let $v$ be the root and set $S_v$ to be some sample for which $A(S_v)=\perp$, add the root to $\V_t$ and remove all other vertices that belong to tree $i$).
        \STATE Set $\s{v_{t+1}}{i'}=\begin{cases} v &i=i'\\ \s{v_{t}}{i'} & i\ne i'\end{cases}$.
        \ELSE \STATE Set $\s{v_{t+1}}{i'}=\s{v_{t}}{i'}$ for all $i'\le k_2$. \ENDIF 
        \STATE Set $\V_{t+1}=\V_t$.
\ENDFOR
\end{algorithmic}
\caption{\dpsoa}\label{alg:stablesoa}
\end{algorithm}
The algorithm, \dpsoa, assumes access to a 
mistake bound algorithm for the class $\cH$ (not necessarily private) such as \soa as in \cite{littlestone1988learning}, which we denote by $A$,\footnote{In particular, $A$ is required to be an algorithm that achieves a mistake bound of at most $d$ on hypothesis classes of Littlestone dimension $d$. We will use the following (easily verified) fact about such an algorithm: after making a mistake, the algorithm must change the hypothesis it outputs for the following round.} as well as call a procedure \svto that is depicted below (\cref{alg:svtorhist}).
We can think of \dpsoa as an algorithm that runs several copies of the same procedure, where each copy is working on its own subsequence of $(x_1, y_1), \ldots, (x_T, y_T)$, and the sub sequences form a random partition of the entire sequence. 

Each process can be described by a tree whose vertices are labelled by samples that are iteratively constructed. Each tree outputs a predictor according to the state of its vertices. Hence, overall the algorithm can be depicted as a forest, where at each iteration an example is randomly assigned to one of the trees, and that tree, in turn, makes an update. 

At each time step, we maintain a set of vertices $\V_t$, which we will call \emph{pertinent} vertices. Each pertinent vertex $v$ holds a sample $S_v$. At time $t = 1$ only the leaves are in $\V_1$, and each leaf $v$ is assigned the sample $S_v=\emptyset$. Then, at every time-step where an example $(x_t, y_t)$ is assigned to the tree, it is randomly assigned to a pertinent vertex $v$ in $\V$ (in detail, it is first randomly assigned to a leaf and then propagated to a pertinent ancestor),
and the sample $S_v$ is updated to $(S_v,(x_t,y_t))$. After that, as we next describe, a process starts that updates the set of pertinent vertices; 
this process follows the idea of the tournament examples presented in \cite{bun20equivalence}.

Whenever two siblings $v,\sib(v)$ are pertinent and assigned with sequences $S_v$ and $S_{\sib(v)}$, respectively, they stay pertinent as long as $A(S_v)=A(S_{\sib(v)})$, and samples are assigned to them at their turn via the process depicted above. Whenever it becomes the case that $A(S_v) \neq A(S_{\sib(v)})$, let $\bar v$ denote the parent of $v, \sib(v)$; we consider an example $x_{\bar v}$ on which $A(S_v),\ A(S_{\sib(v)})$ disagree, and guess its label $y_{\bar v}$. Then, $v, \sib(v)$ are removed from the set of pertinent vertices, their parent $\bar v$ becomes pertinent, and we set $S_{\bar v}$ to equal  $(S_v,(x_v,y_v))$ if $A(S_v)[x_v]\ne y_v$, and $(S_{\sib(v)},(x_v,y_v))$ otherwise. Once this procedure finishes, the tree outputs (randomly) some hypothesis $h=A(S_v)$ where $v$ is a pertinent vertex. The hypothesis will change only when the state of the tree changes (note that at initialization, the tree outputs $A(\emptyset)$). 

\begin{algorithm}
\begin{algorithmic}
\STATE \textbf{Initialize:} parameters $\epsilon,\eta,\delta,c$.
\STATE Let $\sigma=2c/(k\epsilon)$, $\theta=1-3\eta/32$
\STATE Let $\theta_0=\theta+\lap(\sigma)$.
\STATE Let $\counter=1$
\STATE For list $L_1$ set $h_1=\hist_{\epsilon/(2c,\delta/c,\eta)}(L_1)$.
\FOR{$t=1,\ldots, T$:}
\STATE Define  query:
$Q_t=1-\freq_{L_t}(h_{t-1}).$
\STATE Let $\nu_i=\lap(2\sigma)$
\IF{$Q_t+\nu_i \ge \theta_{\counter}$}
    \STATE Set $h_t=\hist_{\epsilon/(2c),\delta/c,\eta}(L_t)$
    \STATE $\counter=\counter+1$
    \STATE $\theta_\counter=\theta+\lap(\sigma)$.
\ELSE
    \STATE Set $h_t=h_{t-1}$
\ENDIF
    \IF{$\counter\ge c$} \STATE ABORT \ENDIF
\ENDFOR
\end{algorithmic}
\caption{\svto: Receives a sequence of $1$-sensitive lists $L_1(D),\ldots, L_T(D)$.}\label{alg:svtorhist}
\end{algorithm}

\subsection{Technical Overview}
We next give a high level overview of our proof techniques. We focus until the end of this section on the oblivious realizable case. The main procedure of the algorithm, \dpsoa, is  \cref{alg:stablesoa}. 

Our proof strategy is similar to the approach of \citet{bun20equivalence} for learning privately in the stochastic setting, which we next briefly describe. In the stochastic setup, the idea was to rely on \emph{global stability}. In a nutshell, a randomized algorithm is called globally stable if it outputs a certain function with constant probability (
over the random bits of the algorithm as well as the random i.i.d sample). Once we can construct such an algorithm (with sufficiently small error) 
we run several copies of the algorithm on separate samples, and then we can use any mechanism, such as the one in Theorem \ref{thm:hist} below, that publishes (privately) an estimated histogram of the frequency of appearance of each function. In detail, given a list $L=\{x_1,\ldots, x_k\}$ we denote by $\freq_L$ the mapping
\[ \freq_L(f) = \frac{1}{k} \sum_{x\in L} \mathbf{1}[x=f].\] 
\begin{theorem}[\cite{bns15} essentially Proposition 2.20]\label{thm:hist}
For every $\epsilon,\delta$ and $\eta$, there exists a $(\epsilon,\delta)$-DP mechanism $\hist_{\epsilon,\delta,\eta}$ that given a list $L=\{x_1,\ldots, x_k\}$, outputs a mapping $\efreq_L:\X\to [0,1]$ such that if
\begin{equation}\label{eq:hist} k\ge \consthist:= 4/\eta+\frac{\log 1/(\eta^2\beta \delta)}{\eta \epsilon}=O\left(\frac{\log 1/\eta\beta\delta}{\eta \epsilon}\right),\end{equation}
then with probability $(1-\beta)$:
\begin{itemize}
    \item If $\efreq_L(x)>0$ then $\freq_L(x)> \frac{\eta}{4}$.
    \item For every $x$ such that $\freq_L(x)> \eta$, we have that $\efreq_L(x)>0$.
\end{itemize}
\end{theorem}

Our algorithm follows a similar strategy but certain care needs to taken due to the sequential (and distribution-free) nature of the data, as well as the fact that using \hist procedure $T$ times may be prohibitive (if we wish to obtain logarithmic regret). We next review these challenges:

\myparagraph{Global Stability}  Our first task is to construct an online version of a globally stable algorithm, which roughly means that different copies of the same algorithm run on disjoint subsequences of $(x_1,y_1), \ldots, (x_T, y_T)$, and output a fixed hypothesis which may depend on the whole sequence but not on the disjoint subsequences. \dpsoa does so by assigning each subsequence to a tree which is running the procedure described in \cref{sec:alg}. We now explain how this procedure induces the desired stability.

As in \cref{sec:alg}, recall that a vertex $v$ is pertinent if it is in the set $\V_t$. We will refer to the distance of a vertex to any of its leaves as that vertex's \emph{depth}. 
Note that for each pertinent vertex $v$ at depth $k$, the algorithm makes $k$ mistakes on the sequence $S_v$ -- indeed, whenever a vertex $\bar v$ is made pertinent, we always append to $S_{\bar v}$ an example which forces a mistake for the sequence of a child of $\bar v$. Also, notice that with probability $2^{-2k_1}$, where $k_1$ is the number of leaves in the tree, all sequences assigned to each pertinent vertex are consistent with the realized hypothesis $h^\star$ (recall that we are considering here the oblivious realizable case, hence $h^\star$ is well-defined). Indeed, this is true as as long as we guessed the label $y_{\bar v}$ to equal $h^\star(x_{\bar v})$ at each round; 
the number of guesses is bounded by the number of vertices, which is $2k_1-1<2k_1$. Ultimately, this allows two cases: in the first case a vertex of depth $d$ is pertinent: in this case the vertex must identify $h^\star$ (indeed, if there are two different hypotheses that are consistent on a sample with $d$ mistakes, then we can force a $(d+1)$th mistake). So, if there are ``many" trees with a $d$-depth pertinent vertex, then fraction of $2^{-2k_1}$ of them, are outputting $h^\star$, hence we found a frequent hypothesis. The second case is that in ``many" of the trees, for some $k<d$, there are many pairs $v,\sib(v)$ of pertinent vertices at depth $k$ so that $A(S_v)=A(S_{\sib(v)})$; we will refer to such a pair $v, \sib(v)$ as a \emph{collision}. 

In the batch case the latter case immediately implies that some hypothesis is outputted frequently (i.e., we get global stability) through a standard concentration inequality that relates the number of collisions between i.i.d random variables, and the frequency of the most probable hypothesis. In the online case it is a little bit more subtle as the examples are not i.i.d, hence the sequences for the pertinent vertices are not i.i.d copies of some random variable.
%
However, suppose that there are many collisions at depth $k$, and that we now reassign the data by randomly permuting the $k$-depth subtree (i.e. we reassign a random parent to each vertex at depth $k$, in order to form a new complete binary tree, and we don't change relations at other depths). Since the assignment of the data $(x_t, y_t)$ to the leaves is invariant under permutation, we can think of this process as randomly picking a new assignment, conditioning on the $k$-th level structure of the trees. 
Alternatively, we can also think of this process as randomly picking without replacement the different hypotheses outputed by the $k$-depth vertices, and counting collisions of siblings. 

We now want to relate the number of collisions to their expected mean and obtain a bound on the most frequent hypothesis. We can do this using a variant of Mcdiarmid's inequality for permutations -- or sampling without replacement. The observation for this inequality was found in \cite{120163} which attributes it to \citet{talagrand1995concentration}. For completeness we provide the proof in \fref{prf:mcreplacement}.
\begin{lemma}[Mcdiarmid's without replacement]\label{lem:mcreplacement}
Suppose $\bar Z=(Z_1,\ldots, Z_n)$ are random variables sampled uniformly from some universe $\Z=\{\s{z}{1},\ldots, \s{z}{N}\}$ without replacement (in particular $n\le N$). Let $F: Z^n \to [0,1]$ be a mapping such that for $\bar z=(z_1,\ldots, z_n)$ and $\bar z'=(z'_1,\ldots, z'_n)$ that are of Hamming distance at most $1$, $|F(\bar z)- F(\bar z')|\le c.$
Then:
\[ \P\left(\E(F(\bar Z))-F(\bar Z)\ge \epsilon\right)\le e^{-\frac{2\epsilon^2}{9nc^2}}.\]
\end{lemma}
We use \cref{lem:mcreplacement} as follows: our function $F$ counts the number of collisions between depth $k$ vertices after a random permutation (where we think here of permutation as sampling without replacement), this function is $1$-sensitive to changing a single element, as required. We thus obtain an estimate of the number of collisions for a random permutation, which we can relate to the appearance of the most frequent hypothesis.

The above calculation can be used to obtain a guarantee that there exists an hypothesis that appears at frequency $2^{-O(k_1)}$ (this frequency is roughly the probability that the tree remains consistent with $h^\star$). Since the number of leaves is exponential in the depth, and the depth needs to be at least $d$ (the upper bound on the level at which the algorithm stabilizes for sure), we overall obtain doubly exponential dependence of the frequency on the Littlestone dimension.

\myparagraph{Mistake Bound}
We next turn to bound the number of mistakes. The crucial observation is that every time the algorithm makes a mistake, if example $x_t$ is assigned to tree $i$ then with some positive probability (specifically, the frequency of $h_t$, lower bounded by $2^{-O(2^{d})}$) tree $i$ outputs $h_t$. Moreover, with probability $1/k_1>0$, $x_t$ is assigned to the pertinent vertex that made the mistake. Once the example is assigned to this vertex, we have $A((S_v, (x_t,y_t)))\ne A(S_{\sib(v)})$. In particular, the two siblings are taken out of the list of pertinent vertices, and their parent becomes pertinent. In other words, every time the algorithm makes a mistake with some constant probability (roughly $2^{-\tilde O(2^{d})}$), the set of pertinent vertices diminishes by one. Since we start with finite number of leaves as pertinent vertices, the expected number of mistakes is bounded by the number of leaves in the forest.

It remains to show that the number of leaves in the forest is logarithmic in the sequence size (but doubly exponential in the Littlestone dimension). The number of leaves is roughly $k_1$ (which is roughly $O(2^d)$) times the number of trees in the forest; this number of trees depends on the sample complexity of the private process in which we output the frequent hypothesis. We now explain why roughly $O(2^{O(2^d)}\ln T)$ trees is sufficient. 

\myparagraph{Online publishing of a globally stable hypothesis} The next challenge we meet is to output the frequent hypothesis. The most straightforward method to do that is to repeat the idea in the batch setting and use procedure \hist. We can guarantee a $O(\sqrt{T})$ factor of deterioration in the privacy parameter $\ep$ (see \cref{lem:composition}) due to the repeated use of the \hist procedure $T$ times.

Our main observation though, is that in most rounds, the frequent hypothesis does not change, allowing us to exploit the \emph{sparse vector technique} \cite{dwork2009complexity}, (see also \cite{dr14}). The sparse vector technique is a method to answer, adaptively, a stream of queries where: whenever the answer to the query does not exceed a certain threshold the algorithm returns a negative result but \emph{without} any cost in privacy. We pay, though, in each round where the query exceed the threshold.

We will exploit this idea in the following setting: 
we receive a stream of $1$-sensitive lists $L_1(S),\ldots, L_T(S)$: Namely, each list $L_t$ is derived from the data $S = \{ (x_1, y_1), \ldots, (x_T, y_T)\}$, and $L_t$ changes by at most one element, given a change in a single $(x_t, y_t)$. 
We assume that at each iteration $t$ we want to output an element $h_t\in L_t$ with high frequency. Our key assumption is that the lists are related and a very frequent element $h_t$ is also frequent at step $t+1$. Thus in most rounds we just verify that $\freq_{L_t}(h_{t-1})$ is large, and only in rounds where it is too small do we use the stable histogram mechanism, paying for privacy.

Indeed, in our setting, the appearance of the frequent hypothesis may diminish by at most one each round. Once its frequency has diminished by a certain factor, 
then we have already made a certain fraction of the maximum possible number of mistakes. Thus, in general we only need to verify that the frequency of $h_{t-1}$ in $L_t$ is sufficiently large each round, which can be done via the sparse vector technique without loss of privacy. We next state the result more formally, the proof is provided in \fref{prf:histsparse}

\begin{lemma}\label{lem:histsparse}
Consider, the procedure $\mathrm{\svto}_{\eta,c,\epsilon}$ depicted in \cref{alg:svtorhist}.
Given a sample $S$, suppose \cref{alg:svtorhist} receives a stream of lists, where each list is a function of $S$ to an array of elements and each list is $1$-sensitive. Then \cref{alg:svtorhist} is $(\epsilon,\delta)$ differentially private and: Set
\begin{equation}\label{eq:sparse} \constsparse:= \frac{8c(\ln T+\ln 2c/\beta)}{\alpha\epsilon},\end{equation}
and suppose: 

\begin{equation}\label{eq:histsparse} k \ge
\consthistsparse:=
\max\{\constsparse,\consthist\}=\tilde O\left(\frac{c\ln T/\beta\delta}{\eta \epsilon}\right),\end{equation}
The procedure then outputs a sequence $\{h_t\}_{t=1}^T$, where $h_t\in L_t$ such that if
%
    for each list $L_t$ there exists $h$ such that $\freq_{L_t}(h)\ge \eta$
%
then with probability at least $(1-2\beta)$, for all $t\le T$, either the algorithm aborted before step $t$ or

\begin{itemize}
    \item     $\freq_{L_{t}}(h_t) \ge \eta/16.$
    \item If $h_{t-1}\ne h_{t}$:
    \[\freq_{L_{t}}(h_{t-1}) \le \eta/8 \quad \mathrm{and}\quad  \freq_{L_{t}}(h_{t})\ge \eta/4.\]

\end{itemize}
\end{lemma}

\paragraph{Adaptive adversaries}
The proof for the oblivious case relies on the existence of an $h^\star$ that is consistent with the data (and independent of the random bits of the algorithm). In the adaptive case, while the sequence has to be consistent, $h^\star$ need not be determined, and the consistent hypothesis may depend on the algorithm's choices.

However, to obtain a regret bound, we rely on the standard reduction that shows that a randomized learner against oblivious adversary, can attain  a similar regret against an adaptive adversary (\cite{cesa2006prediction}, Lemma 4.1).
One issue, though, is that \dpsoa uses random bits that are shared through time. Hence for the reduction to work we need to reinitialize the algorithm at every time-step. In this case, though, the assumptions we make for using the sparse vector technique no longer hold. Thus we can run \dpsoa, using \hist (as we no longer obtain any guarantee from \svto), and we require that each output hypothesis will be $O(\epsilon/\sqrt{T},O(\delta/T))$-DP. The privacy of the whole mechanism now follows from $T$-fold composition:

\begin{lemma}\label{lem:composition}(see for example \citet{dr14})
Suppose $(\epsilon',\delta')$ satisfy:
\begin{equation}\label{eq:composition}\delta'=\delta/2T,\quad \textrm{and}  \quad \epsilon'=\frac{\epsilon}{2\sqrt{2T\ln(1/\delta)}}.\end{equation}
Then, the class of $(\epsilon',\delta')$-differentially private mechanisms satisfies $(\epsilon,\delta)$-differentialy privacy under $T$-fold adaptive composition.
\end{lemma}

Unfortunately though, the above strategy leads to a $\sqrt{T}$ factor in the regret.


\section{Proofs}

\subsection{Proof of \cref{thm:main:oblivious}}\label{prf:main:oblivious}

\myparagraph{Privacy:} We begin by proving the privacy guarantees:
\begin{lemma}\label{lem:privacy}
Suppose we run \cref{alg:stablesoa} with parameters $(\epsilon,\delta)$. Then the output sequence $h_1,\ldots, h_t$ is $(\epsilon,\delta)$-DP.

 \end{lemma}

 \begin{proof} Note that at every time step $t$, changing a single element $x_t$ changes at most one element on the list $L_t=\{A(S_{\s{v_t}{i}})\}_{i=1}^{k_2}$ -- specifically, the tree $i$ for which $\pi(t)$ assigns the element $x_t$. Next, note that if we fix the random bits of the algorithm, except for those that are used in the sub-procedure \svto (i.e. $\pi$ and the random guessing $y_v$), then each list is completely determined at step $t$ by the dataset $S$. 
   Indeed, each $S_{\s{v_t}{i}}$ is independent of $h_1,\ldots, h_T$ and the updates of the algorithm are independent of those. As such, we can think of the lists as functions of the dataset $S$.

The prerequisite assumptions for \cref{alg:svtorhist} hold then (see \cref{lem:histsparse}), and by \cref{lem:histsparse}, we have that the list $h_1,\ldots, h_T$ is then $(\epsilon,\delta)$-DP.
\end{proof}
\myparagraph{Utility:}
The core lemma behind our proof is a statement that there exists (at each iteration) a function that is frequently outputted by a fraction of the trees;
the proof is deferred to \cref{prf:Gstability}.

\begin{lemma}\label{lem:Gstability}
Suppose $(x_1,y_1),\ldots, (x_T,y_T)$ is consistent with some hypothesis $h^\star\in \H$. If \begin{equation}\label{eq:Gstability} k_1\ge \max\{2^{d+1},20\}, \quad\mathrm{and}\quad k_2 \ge 2^{8k_1+6}k_1^2\log \frac{5T\log k_1}{\beta} := \constfreq,\end{equation}
then with probability at least $1-\beta$, for all iterations $t\le T$ there exists a predictor $f\ne \perp$ such that: \[\freq_{L_t}(f)\ge \consteta.\]
\end{lemma}
We continue with the proof of \cref{thm:main:oblivious}, assuming \cref{lem:Gstability}. The proof is an immediate corollary of the following utility lemma.
\begin{lemma}\label{lem:utility}
Suppose \cref{alg:stablesoa} is run on a sequence $(x_1,y_1),\ldots, (x_T,y_T)$, and assume that there exists $h^\star\in \H$ such that $h^\star(x_i)=y_i$ for all $i \in [T]$. Then, for $\beta=1/T$, $\eta$ and $c$ as initialized in \cref{alg:stablesoa}, if: 
\[k_1\ge \max\{2^{d+1},20\}, \quad\mathrm{and}\quad k_2\ge
\max\{\consthistsparse,\constfreq\}=\tilde{O}\left(\frac{2^{8\cdot 2^d}}{\epsilon}\ln T/\delta\right).\]
\ignore{

\max\left\{2^{6k_1+4}k_1^2\log \frac{5T\log k_1}{\beta},
\frac{4}{\eta}+ \frac{8c(\ln 2Tc^3/\eta^2\beta \delta)}{(\eta/32)\epsilon}\right\},\] 
}
 the expected number of mistakes the algorithm makes after $T$ rounds is:
\[\E\left[\sum_{t=1}^T\mathbf{1}[h_t(x_t)\ne y_t]\right] \le  \frac{4k^3_1\cdot 2^{2k_1}k_2}{\eta}+1=\tilde{O}\left(\frac{2^{8\cdot 2^d}}{\epsilon}\ln T/\delta\right) .\]
\ignore{And with probability $1-\delta-\beta$
\[\sum_{t=1}^T \mathbf{1}[h_t(x_t)\ne y_t] \le \frac{4k^2_12^\cdot k_2}{\eta} + \sqrt{2T\ln 1/\delta} + \beta T.\]}
\end{lemma}

\paragraph{Proof of \cref{lem:utility}} 
First, setting $\beta=1/T$ we have by assumption that $k_2\ge \constfreq$.
As such, we can turn to \cref{lem:Gstability} and setting $\eta=\consteta$ we have that, with probability $1-1/T$, for each list $L_t$ there is an element $f$ such that $\freq_{L_t}(f)\ge \eta$. We can now apply \cref{lem:histsparse}, to obtain that, overall with probability $1-3/T$: either the algorithm halted, or for each $t$:
\begin{enumerate}
    \item\label{it:1} $\freq_{L_t}(h_t)\ge \eta/16$.
    \item\label{it:2} If $h_{t-1}\ne h_t$, then
    \[ \freq_{L_t}(h_{t-1})\le \eta/8 \quad \mathrm{and}\quad \freq_{L_t}(h_t)\ge \eta/4.\]
\end{enumerate}
Let us denote this event by $E_0$, and we will assume for now on the $E_0$ happened.

Next, we want to show (under $E_0$) that for $c=\constc$, we have that
\[|\{t: \freq_{L_t}(h_{t-1})\le \eta/8\}|\le c.\]
To see the above, let $t$ be a time-step for which $\freq_{L_t}(h_{t-1})\le \eta/8$, but the algorithm did not abort before time-step $t$. Set $t'<t$ be the last iteration where we called $\textrm{\hist}$ procedure (i.e. the last time we updated counter in \svto). Observe that $h_{t-1}=h_{t'}$, and note that by \cref{it:2} we have that $\freq_{L_{t'}}(h_{t'})>\eta/4$. In particular, the Hamming distance between the lists $L_t$ and $L_{t'}$ is at least $\eta\cdot k_2/4$. 

Note that for each $i \in [k_2]$, $\s{v_t}{i}$ is changed between rounds $t$ and $t+1$ only if we run the While loop in \cref{alg:stablesoa} at round $t$. Note also that at each iteration of the While loop, the size of the set $\V_t$ is decreased by $1$ (as we remove two siblings and add their parent). So  $|\V_{t'}|-|\V_t|\ge \eta\cdot k_2/4$. Let $c_t$ be the number of time steps $t' \leq t$ so that $\freq_{L_{t'}}(h_{t'-1})\le \eta/8$. At initialization we have that $|\V_1|=k_2\cdot k_1$; thus, for all $t \geq 1$,
\[ k_2\cdot k_1 -\eta/4\cdot k_2\cdot c_t \ge 0 \Rightarrow c_t\le 4k_1/\eta.\] 

By the choice of $c = 4k_1/\eta$ in \cref{alg:stablesoa}, the algorithm doesn't halt and we have that, under $E_0$,
\begin{equation}\forall t=1,\ldots, T: \label{eq:htgood}\freq_{L_t}(h_t)> \frac{\eta}{16}.\end{equation}

 We next continue to bound the expected number of mistakes conditioned on $E_0$.

Suppose that $\pi(t)$ belongs to the $i$-th tree. Note that $\pi(t)$ is independent of $h_t$ as well as $\V_t$. We have, then, that with probability $1/k_1$, $\pi(t)$ is a descendent of $\s{v_t}{i}$. One can observe, that for every leaf there exists a unique predecessor that belongs to $\V_t$. Overall then, we obtain that with probability $1/k_1$, $\s{v_t}{i}=v_1$. (Recall that $v_1$ is defined in \cref{alg:stablesoa} to be the unique antecedent of $\pi(t)$ that is in $\V_t$.)

Also, because $\freq_{L_t}(h_t)> \eta/16$, with probability $\eta/{16}$ we have $A(S_{\s{v_t}{i}})=h_t$. 
Taken together we have that whenever the algorithm makes a mistake then $A(S_{v_1})$ makes a mistake with probability at least $\eta/(16k_1)$. Therefore

\[\E\left(\mathbf{1}[h_t(x_t)\ne y_t]\mid ~E_0\right)\le \frac{16k_1}{\eta} \E\left(\mathbf{1}[A(S_{v_1})[x_t]\ne y_t\mid~ E_0\right]).\] 

Again, notice that if $A(S_{v_1})$ makes a mistake, 
 we have that $|\V_t|$ is reduced by at least $1$.
(Indeed, in this case we have that both $v_1,v_2\in \V$ by choice of $\s{v_t}{i}$; because we make a mistake, after adding $(x_t,y_t)$ to the sequence $S_{v_1}$, the algorithm disagrees on these two sequences, hence we run at least one iteration of the While loop that reduces the size of $\V_t$ by at least $1$.) 

As before, since at the beginning $|\V_1|=k_2\cdot k_1$:
\begin{align*} \E\left[\sum_{t=1}^T \mathbf{1}[h_t(x_t)\ne y_t]\mid ~ E_0\right]
&\le 
\frac{16k_1}{\eta} \sum_{t=1}^T\E\left(\mathbf{1}\left[A(S_{v_1})[x_t]\ne y_t\mid ~ E_0\right]\right)\\
&=
\frac{16k_1}{\eta}\E\left(\sum_{t=1}^T\mathbf{1}\left[A(S_{v_1})[x_t]\ne y_t\right]\mid~ E_0\right) \\
&\le
\frac{16k_1|\V_1|}{\eta}\\
&= \frac{16k_1^2k_2}{\eta}.
\end{align*}
Hence, we obtain in expectation
\[\E\left[\sum_{t=1}^T \mathbf{1}[h_t(x_t)\ne y_t]\right]\le \frac{k_1^2k_2}{\eta} + \beta T\le \frac{k_1^2k_2}{\eta}+3.\]
\ignore{
Next, to obtain high probability rate, let \[X_t=\sum_{t'=1}^t \mathbf{1}[h_t(x_t\ne y_t)]-\sum_{t'=1}^t\E\left[\mathbf{1}[h_t(x_t\ne y_t)]\right],\] and set $Y_t=\{h_1,\ldots, h_{t}\}$, and note that $X_t$ is a martingale sequence with respect to $Y_t$ 
Hence by standard concentration inequality we have that
\[ \P\left(X_T\ge \frac{k_1^2k_2}{\eta} +\sqrt{\frac{2\log 1/\delta}{T}}\mid~ E_0\right) \le \delta.
}

\subsection{Proof of \cref{thm:main:adaptive}}
We consider the following procedure:
\begin{itemize}
    \item Given $\epsilon,\delta,T$, set $\epsilon',\delta'$ as in \cref{eq:composition}.
    \item At each time-step $t$, run \dpsoa with privacy parameters $(\epsilon',\delta')$, $k_1,k_2$ on the input sequence $S_t = ((x_1, y_1), \ldots, (x_{t-1}, y_{t-1}))$. 
    \item Receive a sequence $h^{(t)}_1,\ldots, h^{(t)}_t$ from \dpsoa and output $h_t=h^{(t)}_t$.
\end{itemize}
Now, we assume $k_1$ and $k_2$ are chosen so that for an oblivious sequence the conditions of \cref{thm:main:oblivious} are met, and hence
\begin{itemize}
    \item Each output $h^{(t)}_1,\ldots, h^{(t)}_t$ is $(\epsilon',\delta')$-DP w.r.t to the input sequence $S_t=((x_1,y_1),\ldots, (x_{t-1},y_{t-1}))$.
    \item For any oblivious sequence of length $T$, we have that the mistake bound is bounded by $O\left(2^{8\cdot 2^d}/\epsilon'\ln T/\delta'\right)$.
\end{itemize}
Now, for privacy we can use \cref{lem:composition}. Consider the setting of privacy against an adaptive adversary as introduced in \cref{sec:opl-setup}. Observe that, by our definition of the adaptive adversary, each time we apply \dpsoa, we apply it on either the sample $S_t^0= (x^0_1,y^0_1),\ldots, (x^0_t,y^0_t)$, or $S_t^1= (x^1_1,y^1_1),\ldots, (x^1_t,y^1_t)$, which can differ by at most one sample. Therefore, since the mechanism that outputs $h_t^{(t)}$ at step $t$ is $(\epsilon',\delta')$-DP, we obtain via \cref{lem:composition} that the above adaptive online classification algorithm is $(\epsilon,\delta)$-DP.

As for utility, the result follows immediately for the standard reduction from an oblivious online learner to an adaptive one (Lemma 4.1 in \cite{cesa2006prediction}). Indeed, note that at step $t$ we predict $h_t$ according to a distribution $p_t$ which is completely defined by the previous sequence of examples $(x_1, y_1), \ldots, (x_{t-1}, y_{t-1})$ (
it is the distribution from which the oblivious algorithm \dpsoa chooses its prediction). Thus the precondition of \cite[Lemma 4.1]{cesa2006prediction} is verified, and we obtain the regret bound:
\[\sum_{t=1}^T \E[\mathbf{1}[h_t(x)\ne y]] \leq O\left(2^{8\cdot 2^d}/\epsilon'\ln T/\delta'\right). \]

\subsection{Proof of \cref{lem:Gstability}}\label{prf:Gstability}
Let $h^\star$ be a fixed hypothesis that is consistent with the dataset $(x_1, y_1), \ldots, (x_T, y_T)$. We will call a tree $T$ in the forest $G$ consistent if for every vertex $v$, $S_v$ is consistent with hypothesis $h^\star$ and we let $\Tc$ be the sub-graph that consists only of consistent trees. With these notations in mind, we now proceed to the proof. We will divide the proof into two claims; the first one, \cref{cl:consistent}, gives a lower bound on the number of consistent trees.
\begin{claim}\label{cl:consistent}
For a fixed time-step $t\le T$, with probability at least, $1-e^{-\frac{1}{2}k_2\cdot 2^{-4\cdot k_1}}$, we have that $2^{-2k_1-1}\cdot k_2$ of the trees in $G$ are consistent.
\end{claim}
\begin{proof}
Note that for a tree to be consistent we only need that for every $y_{\bar v}$ that we guess while running the algorithm, we have that $y_{\bar v}=h^\star (x_{\bar v})$. If this happens, then all sequences $S_{\bar v}$ remain consistent in the tree. For each $\bar v$, this happens with probability $1/2$, independent on the sequence and the other labels $y_{\bar v}$. Hence each tree is consistent with probability at least $2^{-2\cdot k_1}$ (the number of vertices) and this is independent of the other trees. Thus, applying the Chernoff bound, we obtain that if $M_{t}$ is the number of conistent trees at time $t$, then:

\begin{equation}\label{eq:consistent}  \P\left(M_t\le (2^{-2k_1}-2^{-(2k_1+1)})\cdot k_2\right)\le e^{-2k_2\cdot 2^{-2(2\cdot k_1+1)}}
\end{equation}

\ignore{on a random variable $X_i=1$ iff tree $i$ is consistent we have that for every $t$:

\begin{equation}\label{eq:consistent} \P\left(\frac{1}{k_2}\sum X_i < 2^{-2\cdot k_1}-2^{-2\cdot k_1-1}\right) \le e^{-2k_2\cdot 2^{-2(2\cdot k_1+1)}}.\end{equation}}
\end{proof}
The next step is to prove that (with high probability) there exists a function $f$ that appears frequently in the list $\{A(S_v)\}$ of vertices that belong to consistent trees, which we do next.

First let us denote by $\Xi=(\pi,\{y_v\}_{v \in V})$ the random seed, or internal bits, of \dpsoa, not including the random bits of the mechanisms \svto. Note that, at each time-step, the sets $S_{v}$, and $\V_t$ are completely determined by $\Xi$ (and the oblivious sequence). In particular, the state of the forest is completely independent of the output hypotheses picked by \svto.

Let $\Tc(\Xi;t)$ denote the subgraph of consistent trees given $\Xi$ at time $t$ and let $F_k(\Xi;t)$ be the multiset that consists of all labeled subtrees (at time step $t$) of consistent trees whose root is a depth-$k$ vertex. We will often, with slight abuse of notation, associate a tree in $F_k(\Xi;t)$ to its $S_v$-labeled root $v$, which is a depth-$k$ vertex of some consistent tree; thus we will write, at times, ``for each $v$ in $F_k(\Xi;t)$''. (Also note that it may be the case that for some depth-$k$ vertices $v$, $S_v = \emptyset$; the subtrees rooted at such $v$ are still included in $F_k(\Xi;t)$). 
Also, let us say that the (multi)set $F_k$ is $f$-\emph{heavy} if, for at least  $2^{-k_1}|F_k|$ of the vertices $v$ in $F_k$ we have that $f=A(S_v)\neq \perp$.

Then we have the following claim:

 \begin{claim}\label{cl:F} For a fixed time-step $t\le T$, let $\F$ denote the event that for some $k\le \log k_1+1$ and $f$, $F_k$ is $f$-heavy. then,
\begin{equation}\label{eq:light}\P\left(\F\right)\ge 
1- 2\log k_1 \cdot e^{-\frac{2^{-4k_1-1}}{9}k_2}.\end{equation}

\end{claim}
\begin{proof}

The crucial observation is that, because the distribution of $\pi$ is invariant under permutation of the leaves, then given $F_k$ and $\Tc$, the distribution $\pi$ of the assignments of data points can be viewed as randomly sampling (without replacement) elements from $F_k$ and assigning to each subtree its appropriate depth-$k$ vertex as a root. 

Specifically, let us say that a vertex $v$ is \emph{active} if it belongs to a consistent tree. Now, let $V_k(\Xi;t)$ be the set of labeled depth-$k$ active vertices which are right-children of their parents. 
For each $v \in V_k(\Xi;t)$, denote by $X_v$ the random variable defined as follows: $X_v=1$ if $A(S_v)=A(S_{\sib(v)})$ and $v,v_{\sib(v)}\in \V_t$ at the end of the While loop at step $t$ of \cref{alg:stablesoa}, and $X_v = 0$ otherwise (recall that $t$ is fixed). And further, denote
\[ E(\Xi;k) =  \E\left[\frac{1}{|V_k|}\sum_{v\in V_k}X_v ~\mid ~F_k(\Xi;t),\Tc(\Xi;t)\right].\]

We claim the following bound holds for the time-step $t$:

\begin{equation}\label{eq:mcused} \Pr_{\Xi}\left(\max_k\left\{  E(\Xi;k) -\frac{1}{|V_k|}\sum_{v\in V_k} X_v\right\}> 2^{-k_1} \right)
\le \log k_1 e^{-\frac{2^{-2k_1}|V_k|}{18}}
.\end{equation}
To establish \cref{eq:mcused}, note that for a fixed $k\le \log k_1$  and a set $F_k$, by symmetry of the distribution of $\pi$, the joint distribution of all $X_v$ does not  change if we resample the labels $S_v$ for all vertices $v$ in $F_k$, from this set of all labels, without replacement. Note that changing a single element $S_v$ will change at most one random variable $X_v$, and as such we get that $\frac{1}{|V_k|}\sum X_v$ is $\frac{1}{|V_k|}$-sensitive. Since we randomly draw $2|V_k|$ elements, we can thus use \cref{lem:mcreplacement} to obtain that for a fixed $k$, $F_k$ and $\Tc$:
\begin{align*} &\Pr_{\Xi}\left(  E(\Xi;k) -\frac{1}{|V_k|}\sum_{v\in V_k} X_v> 2^{-k_1} ~\mid~ F_k(\Xi;t)=F_k, \Tc(\Xi;t)=\Tc\right)\\
=&
 \Pr_{\Xi}\left(  \E\left[\frac{1}{|V_k|}\sum_{v\in V_k}X_v ~\mid ~F_k,\Tc\right] -\frac{1}{|V_k|}\sum_{v\in V_k} X_v> 2^{-k_1} ~\mid~ F_k, \Tc\right)\\
\le&  e^{-\frac{2^{-2k_1}|V_k|}{18}}
.\end{align*}

\cref{eq:mcused} now follows by taking expectation over $F_k, \Tc$ as well as a union bound over the $\log k_1$ possible values of $k \leq \log k_1$.

We next observe that for any consistent tree there exists a vertex $v$, such that $v,\sib(v)\in \V_t$ and $A(S_v)=A(S_{\sib(v)})$. Indeed, if this is not the case, then one can prove by induction that the tree's root $v_r$ is in $\V$. However, the sequence $S_{v_r}$ makes $\log k_1\ge d+1$ mistakes, which is a contradiction to the consistency of the tree. Then, what we showed so far is that in any consistent tree there exists $v$ such that $X_v=1$. Thus, applying pigeon-hole principle, we obtain that for any $\pi$ there exists a $k\le \log k_1$ such that
\[\frac{1}{|V_k|} \sum_{v\in V_k} X_v \ge \frac{1}{k_1\log k_1} \ge 2^{-k_1+1}.\]
Together with \cref{eq:mcused} we get that, given $\Tc$, with probability at least $1-\log k_1 \cdot e^{-\frac{2^{-2k_1+1}|V_k|}{18}}$, for some $k$ we have that
\[ E(\Xi;k) > 2^{-k_1}.\]
Finally, (where for ease of notation we neglect the dependence of $F_k,\Tc$ in $\Xi$) we have 
\begin{align*} 
E(\Xi;k)
&=\frac{1}{|V_k|}\sum_{v\in V_k}\P\left(A(S_v)=A(S_{\sib(v)})\neq \perp|F_k,  \Tc\right)\\
&= 
\frac{1}{|V_k|}\sum_{v\in V_k}\sum_{f \neq \perp} \P(A(S_{v})=f|F_k, \Tc) \P\left(A(S_{\sib(v)})=f|A(S_v)=f,F_k, \Tc\right)\\
&\le 
\frac{1}{|V_k|}\sum_{v\in V_k}\sum_{f\neq \perp} \P(A(S_{v})=f|F_k, \Tc) \P\left(A(S_{\sib(v)})=f|F_k, \Tc\right)\\ 
&=
\frac{1}{|V_k|}\sum_{v\in V_k}\sum_{f\neq \perp} \left(\P(A(S_{v})=f|F_k, \Tc)\right)^2\\    
&\le 
\frac{1}{|V_k|}\sum_{v\in V_k}\max_{f\neq \perp}\P(A(S_{v})=f|F_k,\Tc)\\
&= \max_{f\neq \perp}\P(A(S_{v_0})=f|F_k,\Tc),
\end{align*}
where the first inequality follows from the fact that $S_v$ are sampled without replacement, hence the distribution for $A(S_{\sib(v)})=f$ given that we already sampled such an element reduces. The last equality follows from the fact that the distribution of $S_v$, conditioned on $F_k, \Tc$, is identical for all $v \in V_k$; in the last line we set $v_0$ to be an arbitrary vertex in $V_k$. 

Finally, using \cref{cl:consistent}, and noting that $|V_k|$ is at least the number of consistent trees, we have that with probability \[1-\log k_1 \cdot e^{-\frac{2^{-2k_1+1}|V_k|}{18}}-e^{-k_2 \cdot 2^{-4\cdot k_1 -1}}\ge 1-2\log k_1 \cdot e^{-\frac{2^{-4k_1-1}k_2}{9}},\]
for some $k$, we have
\[\max_{f\neq \perp} \P(A(S_{v_0})=f|F_k,\Tc)\ge 2^{-k_1},\]
where again $v_0$ is an arbitrary vertex in $V_k$. Since $S_{v_0}$ is sampled uniformly at random from the set of $S_v$ for $v \in F_k(\Xi;t)$, the left-hand side of the above inequality is simply the fraction of $S_v$, for $v \in F_k(\Xi;t)$ for which $A(S_v) = f$. 
In particular, we obtain that $F_k(\Xi;t)$ is heavy.
\end{proof}
The final claim we will need bounds the number of times we have $A(S_v)=A(S_{\sib(v)})=f$ given the $F_k$ is heavy:

\begin{claim}\label{cl:heavy} For a fixed time-step $t\le T$, recall that $\F$ is the event that $F_k$ is $f$-heavy for some $f$ and $k$. Let $E$ be the event that for at least $2^{-2k_1-1}k_2$ of the trees, there exists a vertex $v$ such that $A(S_v)=A(S_{\sib(v)})=f$, then if $k_1\ge 20$:
\begin{equation}\label{eq:heavy}
\P(E|\F) \ge  1-2e^{-\frac{2^{-8k_1-6}}{9}k_2}.
\end{equation}
\end{claim}
\begin{proof}
Fix the set of consistent trees $\Tc$, and assume that the number of consistent trees is at least  $2^{-k_1-1}\cdot k_2$. We can assume that $k_2\ge 2^{2k_1+2}$ (otherwise, since $k_1\ge 20$ the bound is trivial), hence $|F_k|\ge 2^{k_1+1}$, for any $k$ (as $|F_k|$ is bounded below by the number of consistent trees).

Let us condition $\pi$ on the consistent trees $\Tc$ and $F_k$, which we will assume to be $f$-heavy. Again, we use the fact that conditioned on $F_k, \Tc$, the joint distribution of all $S_v$ ($v \in F_k$) is unchanged if we randomly resample each $S_v$-labeled vertex $v$ from $F_k$, without replacement. 
In particular we have that, for any $k$-depth vertex $v$:
\begin{align*} \P(A(S_{\sib(v)})=f|A(S_v)=f,F_k, \Tc)&\ge
2^{-k_1}-\frac{1}{|F_k|}\\
&\ge 
2^{-k_1}-2^{-k_1-1} & |F_k|\ge 2^{k_1+1}\\
&= 2^{-k_1-1}.\end{align*}

For $i \in [k_2]$, we now set $X_i$ to be the random variable defined by: $X_i=1$ if there exists $v$ in the $i$-th tree such that $A(S_{\sib(v)})=A(S_v)=f$, and $X_i = 0$ otherwise. For each consistent tree $i$, and for any depth-$k$ vertex $v$ of tree $i$, using the fact that $F_k$ is $f$-heavy, we have:
\begin{align*}
    \E[X_i|F_k,\Tc] 
    &\ge \P(A(S_v)=A(S_{\sib(v)})=f|F_k, \Tc)\\
    &=\P(A(S_v)=f\mid F_k,\Tc)\cdot \P(A(S_\sib(v)=f \mid A(S_{\sib(v)})=f,\ F_k, \Tc)\\
    &\ge 2^{-k_1}\P(A(S_\sib{v})=f|A(S_v)=f, F_k,\Tc)\\
    &\ge 2^{-k_1}\cdot 2^{-k_1-1}\\
    &\ge 2^{-2k_1-1}.
\end{align*}
So if $2^{-2k_1-1}\cdot k_2$ of the trees are in $\Tc$, i.e. are consistent, we have that 
\begin{align}\label{eq:nogtrees} \E\left[ \frac{1}{k_2}\sum_{i=1}^{k_2} X_i \mid~F_k, \Tc\right] \ge 2^{-2k_1-1}\cdot 2^{-2k_1-1}=2^{-4k_1-2}.\end{align}
We again exploit the fact that changing the label $S_v$ of a single vertex $v$ in a tree changes at most one random variable $X_i$, and use \cref{lem:mcreplacement} to obtain a high probability rate. In particular, for any set of consistent trees $\Tc$ that includes $2^{-2k_1-1}\cdot k_2$ of the trees, and for any heavy $F_k$:

\begin{equation}\label{eq:mcpermute}  \P\left(\frac{1}{k_2}\sum_{i=1}^{k_2}X_i\le 2^{-4k_1-3}\mid F_k,\Tc\right)\le e^{\frac{2^{-8k_1-6}}{9}k_2}.\end{equation}

Finally, we take expectation over heavy $F_k$. Note that $F_k$ determines if $F_j$ is heavy for all $j\le k$, meaning that we may take the expectation of \cref{eq:mcpermute} over only those $F_k$ for which the determined $F_j$ is not heavy for all $j < k$. And by \cref{cl:consistent}, $\Tc$ consists of $2^{-2k_1-1}\cdot k_2$ of the trees with probability at least $1-e^{-k_2\cdot 2^{-4k_1-2}}$. Hence

\[\P(E|\F) \ge 1-e^{-\frac{2^{-8k_1-6}}{9}k_2}-e^{-2^{-4k_1-2}k_2}\ge 1-2e^{-\frac{2^{-8k_1-6}}{9}k_2}.\]
\end{proof}

\paragraph{Concluding the proof of \cref{lem:Gstability}}
We are now ready to conclude the proof of \cref{lem:Gstability}. First note that if a vertex satisfies $A(S_v)=A(S_{\sib(v)})\ne \perp$ then we must have $v\in \V_t$. Indeed, since for both $S_v,S_{\sib(v)}\ne \perp$, they must at some point have been in $\V$ (because every time we initialize $S_v$ we also add $v$ to $\V$). And whenever we take $v$ out of $\V$ then we must also take $\sib(v)$, but we take them out only if $A(S_v)\ne A(S_{\sib(v)})$).

As such, for any fixed $f$, for any tree that contains a vertex $v$ such that $A(S_v)=A(S_{\sib(v)})=f$, with probability at least $1/k_1$ we have that $A(\s{v_t}{i})=f$ (as $\s{v_t}{i}$ is chosen randomly, at each time-step the tree is updated). Now utlizing \cref{cl:F,cl:heavy} we obtain that with probability at least 
\[1-2e^{-\frac{2^{-8k_1-6}}{9}k_2}-2\log k_1 e^{-\frac{2^{-4k_1-1}}{9}k_2}
\ge
1-4\log k_1 e^{-\frac{2^{-8k_1-6}}{9}k_2}
,\]
at least $2^{-2k_1-1}k_2$ of the trees contain a vertex $v$ such that $A(S_v)=A(S_{\sib(v)})=f$ for some $f \neq \perp$ (independent of the tree).

By the Chernoff bound, we obtain that for at least $\frac{2^{-4k_1-2}k_2}{k_1}$ of these trees $i$, we choose $\s{v_t}{i}$ satisfying $A(\s{v_t}{i})=f$, with probability at least $1 - e^{-\frac{2^{-8k_1-4}}{k^2_1}\cdot k_2}$. 

To conclude, for any fixed $t$, with probability at least

\[1-e^{-\frac{2^{-8k_1-4}}{k_1^2}k_2}-4\log k_1 e^{-\frac{2^{-8k_1-6}}{9}k_2} \ge 1-5\log k_1e^{\frac{2^{-8k_1-6}}{k_1^2}k_2},\]
for $\frac{2^{-4k_1-2}}{k_1}$ fraction of the trees $i$ we have $A(S_{\s{v_t}{i}})=f$ for some fixed $f$. The result now follows from a union bound over $t \leq T$.

\subsection{Proof of \cref{lem:histsparse}}\label{prf:histsparse}

\myparagraph{Privacy} For privacy, the proof is verbatim  the proof that \svt is private provided in \cite{dr14} (but instead of publishing the answer to a linear query everytime a threshold is passed, we output a frequent hypothesis). First, we consider the following variant of the procedure \aboveT introduced in  \cite{dr14}:
 \begin{theorem}[\cite{dr14}, Thm 3.26]
There exists a $(\epsilon,0)$-DP procedure, $\text{\aboveT}_{\theta,c,\epsilon}$ (depicted in \cref{alg:above}, that receives an adaptive sequence of queries $Q_1,\ldots,Q_T$ that are $1/k$ sensitive and outputs a list $\{a_t\}_{t=1}^{T}$ such that if:
\begin{equation} k \ge \constsparse:= \frac{8c(\ln T+\ln 2c/\beta)}{\alpha\epsilon},\end{equation} then for any sequence $Q_1,\ldots, Q_T$ such that $|\{t: Q_t(D)\ge \theta-\alpha\}\le c$, with probability $1-\beta$: 
\begin{itemize}
    \item For all $a_i=\top$:
    $Q_i(D) \ge \theta -\alpha .$
    \item For all $a_i=\perp$:
    $Q_i(D) \le \theta +\alpha .$
\end{itemize}
\end{theorem}

\begin{algorithm}[H]
\begin{algorithmic}[h]
\STATE \textbf{Initialize:} parameters $\epsilon,\theta,c$.
\STATE Let $\sigma=2c/(k\epsilon)$
\STATE Let $\theta_0=\theta+\lap(\sigma)$.
\STATE Let $\counter=0$
\FOR{each list $L_t$}
\STATE Receive a $1/k$ sensitive query $Q_t(D)$
\STATE Let $\nu_i=\lap(2\sigma)$
\IF{$Q_t(D)+\nu_i \ge \theta $}
    \STATE output $\top$.
    \STATE Set $\counter=\counter+1$.
    \STATE Let $\theta_\counter=\theta+\lap(\sigma)$
\ELSE
    \STATE output $\perp$
\ENDIF
    \IF{$\counter\ge c$} \STATE ABORT \ENDIF
\ENDFOR
\end{algorithmic}
\caption{\aboveT}\label{alg:above}
\end{algorithm} 
 We observe that \cref{alg:svtorhist} is the adaptive composition of \aboveT, together with the \hist mechanism with parameters $(\epsilon'/(2c),\delta'/(2c))$. Moreover since each list changes by at most one element if we change a single point in the database, we have that the queries $Q_t(D)=1-\freq_{L_t}(h_{t-1})$ are $1/k$ sensitive. Hence by standard composition we obtain that the algorithm is $(\epsilon,\delta)$-DP.

\myparagraph{Utility} As for accuracy, first note that at each round $t$ we choose as a query
\[ Q_t(L_t) = 1-\freq_{L_t}(h_{t-1}).\] By our choice of parameters (and standard union bound), we have that with probability $(1-2\beta)$ the following happens at each round: Whenever the algorithm chooses $h_t=h_{t-1}$ we have that:
\[1-\freq_{L_t}(h_t)= 1-\freq_{L_t}(h_{t-1})=Q_t(D) \le \theta +\eta/32=1-\eta/16 \Rightarrow \freq_{L_t}(h_t)\ge \eta/16,\]
and at each round that the algorithm calls $\hist$ we have by the guarantee of $\hist$ that:
\[\freq_{L_t}(h_t)\ge \eta/4,\]
and moreover
\[1-\freq_{L_t}(h_{t-1})= Q_t(L_t) \ge \theta -\eta/32 =1-\eta/8\Rightarrow \freq_{L_t}(h_{t-1})\le \eta/8.\]

\subsection{Proof of \cref{lem:mcreplacement}}\label{prf:mcreplacement}
The main observation is that if we let $(i,j)$ be the permutation that switches between $i$ and $j$, a uniform randomly chosen permutation can be written as
\[ \pi=(N,a_N)\circ ((N-1),a_{N-1})\circ \ldots \circ (3,a_3)\circ (2,a_2),\]
where each $a_i$ is an independent random variable distributed uniformly on the set $\{1,\ldots, i\}$. An equivalent way to generate $n$ random variables $\bar Z = (Z_1, \ldots, Z_n)$ sampled without replacement from $\Z = \{ z^{(1)}, \ldots, z^{(N)} \}$ is as follows: first choose a permutation $\pi$ uniformly at random, then set $(i_1,\ldots, i_n)=(\pi(N),\ldots,\pi(N-n+1))$, and finally set $(Z_1, \ldots, Z_n) = (z^{(i_1)}, \ldots, z^{(i_n)})$. In particular, the random variable $\bar{Z} = (Z_1, \ldots, Z_n)$ is completely determined by the independent random variables $a_N,\ldots, a_{N-n+1}$. Let us write this mapping from $a_N, \ldots, a_{N-n+1}$ to $Z_1, \ldots, Z_n$ as $(Z_1, \ldots, Z_n) = G(a_N, \ldots, a_{N-n+1})$. Also note that changing a single variable $a_i$ changes at most the position of $3$ elements of $G(a_N, \ldots, a_{N-n+1})$. Hence, via the triangle inequality, we obtain that, for any tuples $\bar a=(a_N,\ldots, a_{N-n+1})$ and $\bar a'=(a'_N,\ldots, a'_{N-n+1})$ that are of Hamming distance at most $1$,
\[|F(G(\bar a))-F(G(\bar a'))|\le 3c.\]
Thus, considering $F \circ G$ as a function of $a_N,\ldots, a_{N-n+1}$, we obtain the desired result via the standard Mcdiarmid's inequality.

\paragraph{Ackgnoweledgements}
The authors would like to thank Uri Stemmer for helpful discussions. N.G is supported by a Fannie \& John Hertz Foundation Fellowship and an NSF
Graduate Fellowship; R.L is supported by an ISF grant no. ~ 2188/20 and by a grant from Tel Aviv University Center for AI and Data Science (TAD) in collaboration with Google, as part of the initiative of AI and DS for social good.
\bibliographystyle{abbrvnat}
\bibliography{bibliography}
\end{document}